\documentclass[11pt]{article}  

\usepackage{fullpage}
\usepackage{algorithm,algorithmicx,algpseudocode}
\usepackage{amsmath}
\usepackage{amssymb}
\usepackage{amsthm}
\usepackage{array}
\usepackage{bbm}
\usepackage{cancel}
\usepackage{color}
\usepackage{cmap}
\usepackage{enumerate}
\usepackage{enumitem}
\usepackage{mathtools}
\usepackage{mathrsfs}
\usepackage{hyperref}
\usepackage{stackrel}
\usepackage{stmaryrd}
\usepackage{url}
\usepackage{verbatim}
\usepackage{wasysym}
\usepackage{wrapfig}
\usepackage{yhmath}
\usepackage[all,cmtip]{xy}

\usepackage{hyperref}

\usepackage{thmtools}
\usepackage{thm-restate}
\usepackage{cleveref}

\usepackage{authblk}

\usepackage{etoolbox}

\newenvironment{keywords}
{\bgroup\leftskip 20pt\rightskip 20pt \small\noindent{\bfseries
Keywords:} \ignorespaces}%
{\par\egroup\vskip 0.25ex}

\newcommand{\citep}[1]{\cite{#1}}

\usepackage[backend=biber,style=alphabetic]{biblatex}
\newtheorem{thm}{Theorem}[section]
\newtheorem*{thm*}{Theorem}

\newtheorem*{prb*}{Problem}

\newtheorem*{ax*}{Axiom}

\newtheorem*{clm*}{Claim}

\newtheorem*{conj*}{Conjecture}
\newtheorem{cor}[thm]{Corollary}

\newtheorem*{df*}{Definition}

\newtheorem*{ex*}{Example}

\newtheorem{lem}[thm]{Lemma}
\newtheorem*{lem*}{Lemma}

\newtheorem*{pos*}{Postulate}

\newtheorem*{pr*}{Proposition}

\newtheorem*{qu*}{Question}

\newtheorem*{rem*}{Remark}

\newcommand{\E}[0]{\mathbb{E}}

\newcommand{\R}[0]{\mathbb{R}}

\newcommand{\be}[0]{\beta}
\newcommand{\ga}[0]{\gamma}
\newcommand{\Ga}[0]{\Gamma}

\newcommand{\ep}[0]{\varepsilon}

\newcommand{\la}[0]{\lambda}

\newcommand{\te}[0]{\theta}

\newcommand{\Om}[0]{\Omega}

\newcommand{\subeq}[0]{\subseteq}

\newcommand{\iy}[0]{\infty}

\newcommand{\rc}[1]{\frac{1}{#1}}
\newcommand{\prc}[1]{\pa{\rc{#1}}}

\newcommand{\fc}[2]{\frac{#1}{#2}}

\newcommand{\pf}[2]{\pa{\frac{#1}{#2}}}

\newcommand{\nb}[0]{\nabla}

\newcommand{\dx}{\,dx}

\newcommand{\ba}[1]{\left[ {#1} \right]}
\newcommand{\bc}[1]{\left\{ {#1} \right\}}

\newcommand{\pa}[1]{\left( {#1} \right)}

\newcommand{\ve}[1]{\left\Vert {#1}\right\Vert}

\newcommand{\wt}[1]{\widetilde{#1}}
\newcommand{\wh}[1]{\widehat{#1}}

\newcommand{\poly}{\operatorname{poly}}

\newcommand{\Tr}[0]{\operatorname{Tr}}

\newcommand{\val}[0]{\text{val}}

\newcommand{\Vol}[0]{\text{Vol}}

\providecommand{\cal}[1]{\mathcal{#1}}
\renewcommand{\cal}[1]{\mathcal{#1}}

\newcommand{\pull}[9]{
#1\ar@/_/[ddr]_{#2} \ar@{.>}[rd]^{#3} \ar@/^/[rrd]^{#4} & &\\
& #5\ar[r]^{#6}\ar[d]^{#8} &#7\ar[d]^{#9} \\}

\newcommand{\cmp}[9]{
\xymatrix{
#1 \ar[r]^{#4}{#5} \ar@/_2pc/[rr]^{#8}_{#9} & #2 \ar[r]^{#6}_{#7} & #3
}
}

\newcommand{\ha}[1]{\ar@{^(->}[#1]}
\newcommand{\ls}[1]{\ar@{-}[#1]}
\newcommand{\sj}[1]{\ar@{->>}[#1]}
\newcommand{\aq}[1]{\ar@{=}[#1]}
\newcommand{\acir}[1]{\ar@{}[#1]|-{\textstyle{\circlearrowright}}}
\newcommand{\acil}[1]{\ar@{}[#1]|-{\textstyle{\circlearrowleft}}}
\newcommand{\ard}[1]{\ar@{.>}[#1]}
\newcommand{\mt}[1]{\ar@{|->}[#1]}
\newcommand{\inm}[1]{\ar@{}[#1]|-{\in}}
\newcommand{\inr}{\ar@{}[d]|-{\rotatebox[origin=c]{-90}{$\in$}}}
\newcommand{\inl}{\ar@{}[u]|-{\rotatebox[origin=c]{90}{$\in$}}}

\newcommand{\sumo}[2]{\sum_{#1=1}^{#2}}

\newcommand{\prodo}[2]{\prod_{#1=1}^{#2}}
\newcommand{\prodz}[2]{\prod_{#1=0}^{#2}}

\newcommand{\beq}[1]{\begin{equation}\llabel{#1}}
\newcommand{\eeq}[0]{\end{equation}}
\newcommand{\bal}[0]{\begin{align*}}
\newcommand{\eal}[0]{\end{align*}}
\newcommand{\ban}[0]{\begin{align}}
\newcommand{\ean}[0]{\end{align}}

\newcommand{\llabel}[1]{\label{#1}\text{\fixme{\tiny#1}}}

\newcommand{\arxiv}[1]{\url{http://www.arxiv.org/abs/#1}}

\newcommand{\vocab}[1]{\emph{#1}} 

\allowdisplaybreaks[2]

\DeclareFontFamily{U}{wncy}{}
    \DeclareFontShape{U}{wncy}{m}{n}{<->wncyr10}{}
    \DeclareSymbolFont{mcy}{U}{wncy}{m}{n}
    \DeclareMathSymbol{\Sh}{\mathord}{mcy}{"58}

\renewcommand{\llabel}[1]{\label{#1}}

\title{Efficient sampling from the Bingham distribution}
\author[1]{Rong Ge}
\author[2]{Holden Lee}
\author[3]{Jianfeng Lu}
\author[4]{Andrej Risteski}

\affil[1]{Duke University, Computer Science Department \authorcr
  \tt rongge@cs.duke.edu}
\affil[2]{Johns Hopkins University, Applied Mathematics and Statistics Department \authorcr
  \tt hlee283@jhu.edu}
\affil[3]{Duke University, Mathematics Department \authorcr
  \tt jianfeng@math.duke.edu}
\affil[4]{Carnegie Mellon University, Machine Learning Department\authorcr
  \tt aristesk@andrew.cmu.edu}
    \date{September 30, 2020\footnote{Updated December 8, 2023.}}
\begin{document}

\maketitle

\begin{abstract}
    We give a algorithm for exact sampling from the Bingham distribution $p(x)\propto \exp(x^\top A x)$ on the sphere $\mathcal S^{d-1}$ with expected runtime of $\operatorname{poly}(d, \lambda_{\max}(A)-\lambda_{\min}(A))$. The algorithm is based on rejection sampling, where the proposal distribution is a polynomial approximation of the pdf, and can be sampled from by explicitly evaluating integrals of polynomials over the sphere. Our algorithm gives exact samples, assuming exact computation of an inverse function of a polynomial. This is in contrast with Markov Chain Monte Carlo algorithms, which are not known to enjoy rapid mixing on this problem, and only give approximate samples.
    
    As a direct application, we use this to sample from the posterior distribution 
    of a rank-1 matrix inference problem in polynomial time.
\end{abstract}



\section{Introduction}

Sampling from a probability distribution $p$ given up to a constant of proportionality is a fundamental problem in Bayesian statistics and machine learning. A common instance of this in statistics and machine learning is posterior inference (sampling the parameters of a model $\theta$, given data $x$), where the unknown constant of proportionality comes from an application of Bayes rule: $p(\theta|x) \propto p(x|\theta) p(\theta)$.

However, for standard approaches to sampling such as the Langevin Monte Carlo algorithm, provable results on efficient (polynomial-time) sampling often require that $p$ be log-concave or close to log-concave.
In this work, we consider the problem of sampling from a specific non-log-concave probability distribution on the sphere $\cal S^{d-1}$ in $d$ dimensions: the \emph{Bingham distribution.} In addition to having applications in statistics, the Bingham distribution is of particular interest as it models the local behavior of any smooth distribution around a stationary point. 

We give a polynomial-time algorithm based on approximating the probability density function by a polynomial and explicitly evaluating its integral over the sphere. Our algorithm is of Las Vegas type: It has the advantage of giving \emph{exact} samples, assuming exact computation of an inverse function of a polynomial. Our approach contrasts with the usual Markov Chain Monte Carlo algorithms, which are not known to enjoy rapid mixing on this problem, and only give approximate samples.

The Bingham distribution \citep{bingham1974antipodally} defined by a matrix $A\in \R^{d\times d}$ 
is the distribution on the sphere $\cal S^{d-1}\subeq \R^d$ whose density function with respect to the uniform (surface) measure is given by 
\begin{align*}
    p(x):=\fc{dP}{d\mu_{\cal S^{d-1}}}(x) &\propto \exp(x^\top Ax).
\end{align*}
Note that due to the symmetric form, without loss of generality, we can assume $A$ is symmetric.
This  distribution finds frequent use in \emph{directional statistics}, which studies distributions over the unit sphere. In particular, the Bingham distribution is widely used in paleomagnetic data analysis \citep{onstott1980application} and has applications to computer vision \citep{antone2000automatic,haines2008belief,glover2013bingham} and even differential privacy \citep{chaudhuri2013near, wang2015differentially}. As shown in Section~\ref{s:rank1}, it also naturally appears in the posterior distribution for a rank-1 matrix inference problem, a special case of matrix factorization.

Our main theorem is given below. In the following, we will identify a probability distribution over $\mathcal S^{d-1}$ with its density function with respect to the uniform measure on $\mathcal S^{d-1}$.

\begin{restatable}{thm}{tmain}
\label{t:main}
\label{t:main-poly}
Let $A$ be a symmetric matrix with maximum and minimum eigenvalue $\la_{\max}$ and $\la_{\min}$, respectively.
Let $p(x)
\propto \exp(x^\top Ax)$ be a probability distribution over $\mathcal{S}^{d-1}$. 
Then, given an oracle for solving a univariate polynomial equation, Algorithm~\ref{a:main} produces a sample from $p(x)$ and runs in expected time $\poly(\la_{\max}-\la_{\min}, d)$. 
\end{restatable}

We can consider the Bingham distribution as a ``model" non-log-concave distribution, because any smooth probability distribution looks like a Bingham distribution 
in a sphere of small radius around a stationary point.\footnote{The more general \emph{Fisher-Bingham distribution} includes a linear term, and so can locally model any smooth probability distribution.} More precisely, suppose $f:\R^d\to \R$ is 3-times differentiable, $p(x)= e^{-f(x)}$ on $\R^d$, and $\nb f(x_0)=0$. Then we have that as $x\to x_0$,
\begin{align*}
    p(x) &= \exp\bc{-[f(x_0) + (x-x_0)^\top (\nb^2 f(x_0)) (x-x_0) + O(\ve{x-x_0}^3)]}.
\end{align*}
Note that if we can sample from small spheres around a point, we can also sample from a small  ball around the point by first estimating and sampling from the marginal distribution of the radius.

Moreover, the Bingham distribution already illustrates the challenges associated with sampling non-log-concave distributions. 
First, it can be arbitrarily non-log-concave, as the minimum eigenvalue of the Hessian can be arbitrarily negative. Second, when $A$ has distinct eigenvalues, the function $f(x) = x^\top Ax$ on $\mathcal S^{d-1}$  has $2(d-1)$ saddle points and 2 minima which are antipodal.
Hence, understanding how to sample from the Bingham distribution may give insight into sampling from more general non-log-concave distributions.

\subsection{Related work}

We first discuss general work on sampling, and then sampling algorithms specific to the Bingham distribution.

Langevin Monte Carlo \citep{rossky1978brownian,roberts1996exponential} is a generic algorithm for sampling from a probability distribution $p(x) \propto e^{-f(x)}$ on $\R^d$ given gradient access to its negative log-pdf $f$. It is based on discretizing Langevin diffusion, a continuous Markov process. In the case where $p$ is log-concave, Langevin diffusion is known to mix rapidly \citep{bakry1985diffusions}, and Langevin Monte Carlo is an efficient algorithm \citep{dalalyan2016theoretical,durmus2016high}. More generally, for Langevin diffusion over a compact manifold (such as $\cal S^{d-1}$), positive Ricci curvature can offset non-log-concavity of $p$, and rapid mixing continues to hold if the sum of the Hessian of $f$ and Ricci curvature at any point is lower bounded by a positive constant \citep{bakry1985diffusions,hsu2002stochastic}. In our setting, this is only the case when the maximum and minimum eigenvalues of $A$ differ by less than $\fc{d-1}2$: $\la_{\max}(A) - \la_{\min}(A) < \fc{d-1}2$. We note there are related algorithms such as Hamiltonian Monte Carlo~\citep{duane1987hybrid} that are more efficient in the log-concave case, but still suffer from torpid mixing in general.

Next, we consider algorithms tailored for the Bingham distribution. 
An important observation is that the normalizing constant of the Bingham distribution is given by the hypergeometric function of a matrix argument \cite[p.182]{mardia2009directional},
\begin{align*}
    \int_{\cal S^{d-1}}\exp(x^\top A x) \,d\cal S^{d-1}(x) &= {}_1F_1\pa{\rc2;\fc n2;D}^{-1}
\end{align*}
where $D$ is the diagonal matrix of eigenvalues of $A$. 
Methods to approximate the hypergeometric function are given in~\cite{koev2006efficient}, however, with super-polynomial dependence on the degree of the term where it is truncated, and hence on the accuracy required. 

The previous work \citep{kent2013new} gives an rejection sampling based algorithm where the proposal distribution is an angular central gaussian envelope, that is, the distribution of a normalized gaussian random variable. This distribution has density function $p(x) \propto (x^\top \Om x)^{-d/2}$ 
for $\Om$ chosen appropriately depending on $A$.
The efficiency of rejection sampling is determined by the maximum ratio between the desired ratio and the proposal distribution. Their bound for this ratio depends on the normalizing constant for the Bingham distribution~\cite[(3.5)]{kent2013new}, and they only give an polynomial-in-dimension bound when the temperature approaches zero (that is, for the distribution $\exp(\be x^\top  Ax)$ as $\be\to \iy$).
Our algorithm is also based on rejection sampling; however, we use a more elaborate proposal distribution, for which we are able to show that the ratio is bounded at all temperatures.

\subsection{Application to rank-1 matrix inference}

\label{s:rank1}
The algorithm we give has an important application to a particularly natural statistical inference problem: that of 
recovering a rank-1 matrix perturbed by Gaussian noise. 

More precisely, suppose that an observation $Y$ is produced as follows: we sample  $x \sim \mathcal{D}$ for a prior distribution $\mathcal{D}$ and $N \sim  \mathcal{N}(0,\gamma^2 I)$, then output $Y = xx^\top  + N$. By Bayes Rule, the posterior distribution over $x$ has the form 
\begin{align}\label{e:rank1-post}
    p(x | Y) &\propto \exp\left(-\frac{1}{2\gamma^2}\|Y - xx^\top \|^2_F\right)p(x).
\end{align}
In the particularly simple case where $\mathcal{D}$ is uniform over the unit sphere, this posterior has the form we study in our paper: 
\begin{align*}
    p(x | Y) &\propto \exp\left(\frac{1}{2\gamma^2} x^\top  Y x\right)
\end{align*}
for $x\in \cal S^{d-1}$. Thus, we are able to do posterior sampling. More generally, for radially symmetric $p(x)$, we can approximately sample from the radial distribution of the marginal, after which the problem reduces to a problem on $\cal S^{d-1}$. Note that our algorithm does not require the model to be well-specified, i.e., it does not require $Y$ to be generated from the hypothesized distribution.

In existing literature, the statistics community has focused more on questions of \emph{recovery} (can we achieve a non-trivial ``correlation'' with the planted vector $x$ under suitable definitions of correlation) and \emph{detection} (can we decide with probability $1-o(1)$ as $d \to \infty$ whether the matrix presented is from the above distribution with a ``planted" vector $x$, or is sampled from a Gaussian) under varying choices for the prior $\mathcal{D}$. In particular, they study the threshold for $\gamma$ at which each of the respective tasks is possible. The two most commonly studied priors $\mathcal{D}$ are uniform over the unit sphere (\emph{spiked Wishart model}), and the coordinates of $x$ being $\pm\frac{1}{\sqrt{d}}$ uniformly at random (\emph{spiked Wigner}). For a recent treatment of these topics, see e.g., \cite{peche2006largest, perry2018optimality}. 

However, the statistical tests involve calculating integrals over the posterior distribution~\eqref{e:rank1-post} (for instance, the MMSE $\wh x \wh x^\top = \fc{\int xx^\top \exp(-\rc{2\ga^2}\ve{Y-xx^\top}_F^2) p(x)\dx}{\int \exp(-\rc{2\ga^2}\ve{Y-xx^\top}_F^2) p(x)\dx}$) , and the question of algorithmic efficiency of this calculation is not considered. Our work makes these statistical tests algorithmic (for spherically symmetric priors), because integrals over the posterior distribution can be approximated through sampling.

On the algorithmic side, the closest relative to our work is the paper by \cite{moitra2020fast}, which considers the low-rank analogue of the problem we are interested in: namely, sampling from the distribution\begin{align*}
    p(X) &\propto \exp\left(-\frac{1}{2\gamma^2}\|XX^\top  - Y\|_F^2\right)
\end{align*}
supported over matrices $X \in \mathbb{R}^{d \times k}$, s.t. $Y = X_0 X_0^\top  + \gamma N$, for some matrix $X_0  \in \mathbb{R}^{d \times k}$ and $N \sim \mathcal{N}(0, I)$. It proves that a slight modification of Langevin Monte Carlo can be used to sample from this distribution efficiently in the \emph{low-temperature} limit, namely when $\gamma = \Omega(d)$.  

For comparison, in this paper, we can handle an \emph{arbitrary} temperature, but only the rank-1 case (i.e. $k=1$). Moreover, the algorithm here is substantially different, based on a polynomial approximation of the pdf, rather than MCMC. Extending either approach to the full regime (arbitrary $k$ and arbitrary temperature) is an important and challenging problem. 

\section{Algorithm based on polynomial approximation} 

We present our rejection sampling algorithm as Algorithm~\ref{a:main}. Our main theorem is the following.

\renewcommand{\algorithmicrequire}{\textbf{Input:}}
\renewcommand{\algorithmicensure}{\textbf{Output:}}

\begin{algorithm}[h!]
\caption{Sampling algorithm for Bingham distribution} 
\begin{algorithmic}[1]
\Require Symmetric matrix $A$ 
\Ensure A random sample $x \sim p(x) \propto \exp(x^{\top} Ax)$ on $\mathcal{S}^{d-1}$
\medskip  
\State Diagonalize $[V, \Lambda] = \mathrm{diag}(A)$ such that $A = V \Lambda V^{\top}$; let $\lambda_{\min}$ and $\lambda_{\max}$ denote the smallest and largest eigenvalues respectively; 
\State Set $D = \Lambda - \lambda_{\min}I_d$;
\State Set $n = (\lambda_{\max} - \lambda_{\min})^2$;
\Repeat \Comment{Rejection sampling for $z \sim \widetilde{p}(z) \propto \exp(z^{\top} D z)$ on $\mathcal{S}^{d-1}$}
\For{$i = 1 \to d$} \Comment{\parbox[t]{.58\linewidth}{Sample proposal $z \sim q(z) \propto \bigl( z^{\top} (I+D/n) z \bigr)^n$ on $\mathcal{S}^{d-1}$ one coordinate at a time}}
    \If{$i = 1$}
        \State Let $D_1 = D$;
        \State Determine the marginal distribution $q(z_1)$ whose pdf is given as follows, where $(D_1)_{-1}$ represents the submatrix of $D_1$ obtained from deleting the first row and column (see Theorem~\ref{t:xDxn}, \eqref{eq:integral} for details)
         $$
 \frac{\pa{1-x^2}^{(d-3)/2}}{Z} 
 \int_{y \in \mathcal{S}^{d-2}} \left((1+ (D_1)_{11}/n) z_1^2 + (1-z_1^2) y^\top  (I_{d-1}+(D_1)_{-1}/n) y\right)^n \,d\mathcal{S}^{d-2}(y) 
 ;$$
        \State Sample $z_1 \sim q(z_1)$ via inverse transform sampling (Lemma~\ref{l:inverse});
        \State Let $y_1=z_1$;
    \Else
        \State Let $D_i = y_{i-1}^2 (D_{i-1})_{11} + (1-y_{i-1}^2) (D_{i-1})_{-1}\in \R^{(d-i+1)\times (d-i+1)}$; 
        
        \Comment{We will sample from the distribution $\propto (y^\top (I+D_i/n) y)^n$.}
        \State Determine the conditional marginal distribution $q(y_i \vert z_1, \ldots, z_{i-1})$ where $z_i = y_i\sqrt{1-\sumo j{i-1}z_j^2}$, whose pdf is given by 
        %
        (see Theorem~\ref{t:xDxn}, \eqref{eq:integral} for details)
        {\small$$
         \frac{\pa{1-x^2}^{(d-i-2)/2}}{Z} 
 \int_{(y_{i+1},\ldots,y_d) \in \mathcal{S}^{d-i-1}} \left((1+ (D_i)_{11}/n) y_i^2 + (1-y_i^2) y^\top  (I_{d-i}+(D_i)_{-1}/n) y\right)^n \,d\mathcal{S}^{d-i-1}(y) 
 ; 
        $$}
        \State Sample $y_i \sim q(y_i \vert z_1, \ldots, z_{i-1})$ via inverse transform sampling (Lemma~\ref{l:inverse});
        \State Let $z_i = y_i\sqrt{1-\sumo j{i-1}z_j^2} $;
    \EndIf
\EndFor
\State Accept $z$ with probability $e^{-1} \frac{\exp(z^\top D z)}{(z^\top (I+D/n)z)^n}$; \Comment{\parbox[t]{.45\linewidth}{Rejection sampling (see proof of Theorem~\ref{t:main-poly} for explanation of the $e^{-1}$ factor)}}
\Until the sample $z$ is accepted;
\State \Return $x = V z$;
\end{algorithmic}
\label{a:main}
\end{algorithm}

\tmain*
Before proceeding to the proof of Theorem~\ref{t:main-poly}, we make a few remarks about the statement. Firstly, we work in the real model of computation. Solving a polynomial equation can be done to machine precision using binary search, so the only errors present when actually running the algorithm are roundoff errors.

The algorithm is based on rejection sampling: we calculate a proposal sample in time $\poly(\la_{\max}-\la_{\min},d)$, accept it with some probability, and otherwise repeat the process. In the parlance of algorithms, this means that it is a Las Vegas algorithm: it produces an exact sample, but has a randomized runtime.
For the analysis, we lower bound the acceptance probability by an absolute constant. The number of proposals until acceptance follows a geometric distribution with success probability equal to the acceptance probability. Hence, the total time is polynomial with high probability.

The analysis of our algorithm proceeds in the following steps:
\begin{enumerate}
    \item By diagonalization and change-of-coordinates, we show that it suffices to provide an algorithm for sampling from distributions over the unit sphere $p: \mathcal{S}^{d-1} \to \mathbb{R}^+$ in the form
$$ p(x) \propto \exp\left(x^\top  D x\right),$$  
where $D \in \mathbb{R}^{d \times d}$ is diagonal and PSD.

   \item We show that if we use $q(x) \propto (x^\top (I + D/n)x)^n$ as a proposal distribution, when $n \geq D_{\max}^2$ the ratio $\max\{\frac{p(x)}{q(x)}, \frac{q(x)}{p(x)}\}$ is bounded by an absolute constant.
   
   \item We then show that CDF for the marginal distributions of $q(x)$ can be computed explicitly in polynomial time (in $n, d$), therefore using inverse transform sampling, one can sample from $q$ in polynomial time.
\end{enumerate}


\paragraph{Change-of-coordinates}
We first argue that it suffices to provide an algorithm for sampling from distributions over the unit sphere $p: \mathcal{S}^{d-1} \to \mathbb{R}^+$ in the form
$$ p(x) \propto \exp\left(x^\top  D x\right)$$  
where $D \in \mathbb{R}^{d \times d}$.
To see this, note that if $A=V D V^\top$ with $D$ diagonal and $V$ orthogonal, then given a sample $x$ from the distribution $\propto \exp(x^\top Dx)$, $V x$ is a sample from the distribution $\propto \exp(x^\top V DV^\top x)$.
Moreover, we can assume that $D$ is a PSD diagonal matrix, with smallest eigenvalue $D_{\min} = 0$ and largest eigenvalue $D_{\max}$. This is because replacing $D$ by $D-cI_d$ simply multiplies $\exp(x^\top Dx)$ by a constant on $\cal S^{d-1}$, and we can take $c=D_{\min}$.

\paragraph{Proposal distribution} Next we give a proposal distribution for rejection sampling based on polynomial approximation of $p$: 
\begin{lem}\label{l:approx}
Let $D\in \R^{d\times d}$ be diagonal with minimum eigenvalue $D_{\min}\ge 0$ and maximum eigenvalue $D_{\max}$. 
Let the distribution $q: \mathcal{S}^{d-1} \to \mathbb{R}^+$ be defined as $q(x) \propto (x^\top (I + D/n)x)^n$ for $n\ge 1$. 
Then, 
$$ \max\left\{\fc{p(x)}{q(x)}, \fc{q(x)}{p(x)}\right\} \leq 
\exp\pf{D_{\max}^2}{2n}.
$$ 
Moreover, if $D_{\min}=0$, letting $v$ be a unit eigenvector with eigenvalue $0$, $1 \le \fc{q(v)}{p(v)}\le \exp(\fc{D_{\max}^2}{2n})$.
\end{lem}
Note that only an upper bound on $\frac{p(x)}{q(x)}$ is necessary for rejection sampling; however, the lower bound 
comes for free with our approach. 
The assumption $D_{\max}\ge0$ is simply for convenience; in general we can replace $D_{\max}$ by $D_{\max}-D_{\min}$.
\begin{proof} 
First, we show that 
\begin{equation} 
-\fc{D_{\max}^2}{2n} \le n\log( x^\top  (I+D/n)x) - x^\top  D x \le 0.
\label{eq: unnormalizedbd} \end{equation}
By Taylor's theorem with remainder, we have for $x\in \mathcal{S}^{d-1}$ that 
\begin{align*}
    \log(x^\top (I+D/n)x) 
    &= \log (1+x^\top Dx/n)\\
    &=\fc{x^\top Dx}n - \rc 2 \rc{(1+\xi)^2} \pf{x^\top Dx}n^2 & \text{for some }\xi\in [0,x^\top Dx/n].
\end{align*}
Because $\ve{x}=1$, we have $x^\top Dx/n \le D_{\max}/n$, so 
\begin{align*}
    \log(x^\top (I+D/n)x)  &\in \ba{\fc{x^\top Dx}{n} - \fc{D_{\max}^2}{2n^2}, \fc{x^\top Dx}{n}}
\end{align*}
Multiplying by $n$, \eqref{eq: unnormalizedbd} follows. 
Now \eqref{eq: unnormalizedbd} implies by exponentiation that 
\begin{align*}
    \exp\pa{-\fc{D_{\max}^2}{2n}}\le \frac{(x^\top  (I + D/n)x)^n}{\exp(x^\top  D x)}
    \le 1
\end{align*}
and hence
\begin{multline*} 
\exp\pa{-\fc{D_{\max}^2}{2n}}\le \left.
\fc{(x^\top  (I + D/n)x)^n}{\int_{\mathcal{S}^{d-1}}(x^\top  (I + D/n)x)^n\,d\mathcal{S}^{d-1}(x)}
\right/
\fc{\exp(x^\top  D x)}{\int_{\mathcal{S}^{d-1}} \exp(x^\top  D x)\,d\mathcal{S}^{d-1}(x)} \\
\leq \exp\pa{\fc{D^2_{\max}}{2n}} \end{multline*}  
from which the lemma immediately follows. 

For the last statement, note that for $x=v$, the numerators $(x^\top  (I + D/n)x)^n$ and $\exp(x^\top  D x)$ in the above expression both equal 1.
\end{proof} 

\paragraph{Sampling from proposal $q$} Finally, we show that it is possible to sample from $q(x)$ efficiently in time polynomial in $n, d$. First we show that the high order moments for quadratic forms can be computed efficiently.


\begin{lem}[Calculating integrals of quadratic forms] The integral 
$$ \int_{\mathcal{S}^{d-1}} (x^\top  D x)^n\,d\mathcal S^{d-1}(x) $$ 
can be calculated in time $\mbox{poly}(n,d)$. 
\label{eq:integralquadratic} \label{l:integral}
\end{lem}
\begin{proof} 
The result follows essentially from known formulas about moments of quadratic functions under a Gaussian distribution. 

First, we show the task reduces to calculating 
$$\mathbb{E}_{x \sim N(0,I_d)} [(x^\top  Dx)^n].$$ 
A Gaussian can be sampled by sampling the norm of $x$ and the direction of $x$ independently.  
Hence, 
\begin{align}
\label{e:norm-dir}
\mathbb{E}_{x \sim N(0,I_d)}[ (x^\top  Dx)^n] &= \mathbb{E}_{x \sim N(0,I_d)} 
[\|x\|^{2n} ]
\cdot 
\mathbb{E}_{x \sim N(0,I_d)}
\ba{\left(\pf{x}{\|x\|}^\top  D\pf{x}{\|x\|}\right)^n }.
\end{align} 
The second factor is (up to a constant) the integral of interest as $\frac{x}{\|x\|}$ is uniformly distributed over the sphere:
\begin{align*}
\E_{x\sim \mathcal{S}^{d-1}} [(x^\top D x)^n]
&=
\fc{\int_{\mathcal{S}^{d-1}} (x^\top Dx)^n\,d\mathcal{S}^{d-1}(x)}{\Vol(\mathcal{S}^{d-1})}
=\fc{\int_{\mathcal{S}^{d-1}} (x^\top Dx)^n\,d\mathcal{S}^{d-1}(x)}{2\pi^{d/2}/\Ga(d/2)}.
\end{align*}
The first factor in~\eqref{e:norm-dir} has a simple closed-form expression given by Corollary~\ref{c:x2n}. 

Thus it remains to calculate the LHS of~\eqref{e:norm-dir}, the expectation under the Gaussian. We use the recurrence from \cite{kan2008moments}, reprinted here as Corollary~\ref{c:xTAx-recursion}: denoting $S(n) = \frac{1}{n! 2^n} \mathbb{E}_{x \sim N(0,I_d)} [(x^\top  Dx)^n]$, we have $S(0)=1$ and for $n\ge 1$,
\begin{equation}
    S(n) = \frac{1}{2n}\sum_{i=1}^n \mbox{Tr}(D^i) S(n-i)
\end{equation}
which can be calculated in time $\mbox{poly}(n,d)$ by dynamic programming. 
\end{proof}

Using this integral, we can compute the unnormalized cdf for the marginals of distribution $q$. This can then be combined with the technique of \vocab{inverse transform sampling}. 

\begin{lem}[Inverse transform sampling]\label{l:inverse}
Suppose that we know that the probability distribution on $[a,b]$ has pdf $p(x)\propto f(x)$, and we can calculate the (unnormalized) cdf $F(x)=\int_a^x f(t)\,dt$. Then given an oracle for computing the inverse of $G(x) = F(x)/F(b)$, one can sample from the distribution.
\end{lem}

\begin{proof}
The algorithm simply generates a uniformly random number $r\in [0,1]$ and computes $G^{-1}(r)$. Since $G(x)$ is the cdf of the probability distribution we know $G^{-1}(r)$ is exactly a random variable from this probability distribution $p(x)$.
\end{proof}

Note that when the cdf $F(x)$ is a polynomial, it is possible to compute $G^{-1}$ with accuracy $\ep$ in $\poly\log(1/\ep)$ time by binary search.  

Combining Lemma~\ref{l:integral} and \ref{l:inverse} we are ready to show that one can sample from $q(x)$ efficiently.

\begin{thm} \label{t:xDxn}
Let $D$ be a diagonal PSD matrix and let $q(x) \propto (x^\top (I+ D/n) x)^n$. 
Given an oracle for solving a univariate polynomial equation, we can sample from $q(x)$ in time $\poly(n,d)$.
\end{thm} 
As suggested above, we can solve the polynomial equation using binary search, obtaining an $\ep$-accurate solution using $\poly\log\prc{\ep}$ evaluations of the polynomial.
\begin{proof}
Note the theorem is trivial for $d=1$, as $q(x)$ is the uniform distribution on $\cal S^0=\{-1,1\}$. Hence we assume $d>1$.

We will sample the coordinates one at a time (see Algorithm~\ref{a:main}). 
For notational convenience, let us denote by $x_{-i}$ the set of coordinates of a vector $x$ excluding the $i$-th.  

Namely, we will show that: \begin{enumerate} 
\item We can efficiently sample 
from the marginal distribution of $x_1$, denoted by\footnote{This is a slight abuse of notation, and it denotes the marginal probability of the first coordinate. We do this to reduce clutter in the notation by subscripting the appropriate coordinate.} $q(x_1)$, 
via inverse transform sampling. To do this, we exhibit a $\mbox{poly}(n,d)$ algorithm for calculating the CDF of $q(x_1)$. 
\item For any $x_1$, the conditional distribution $q(x_{-1}| x_1)$ also has the form $q(x_{-1} | x_1) \propto (x_{-1}^\top  (I+\wt{D}/n) x_{-1})^{n}$, for some diagonal PSD matrix $\wt{D} \in \mathbb{R}^{(d-1) \times (d-1)}$. 
\end{enumerate}
Applying this recursively gives our theorem.

Towards proving part 1, the marginal can be written 
using the co-area formula as 
\begin{align*} q(x_1) &= \frac{(1-x_1^2)^{-(d-1)/2} \int_{\sqrt{1-x_1^2}\mathcal{S}^{d-2}} q(x_1, x_{-1}) \,d\mathcal{S}^{d-2}(x_{-1})}{Z} \\ 
&=\frac{(1-x^2)^{-(d-1)/2} \int_{\sqrt{1-x_1^2}\mathcal{S}^{d-2}} \left((1+ D_{11}/n) x^2_1 + \sum_{i=2}^d (1+D_{ii}/n) x^2_i\right)^n d\mathcal{S}^{d-2}(x_{-1})}{Z},
\end{align*}
where 
$Z = \int_{\mathcal{S}^{d-1}} (x^\top  (I + D/n)x)^n$. 

Introducing the change of variables $x_{-1} = y \sqrt{1-x_1^2}$ where $y=(y_2,\ldots ,y_d)\in \mathcal S^{d-2}$, we can rewrite the numerator as 
\begin{align*} 
\pa{1-x_1^2}^{(d-3)/2}
\int_{\mathcal{S}^{d-2}} \left((1+ D_{11}/n) x^2_1 + (1-x_1^2)\sum_{i=2}^d (1+D_{ii}/n) y^2_i\right)^n d\mathcal{S}^{d-2}(y).  
\end{align*} 
Hence, the CDF for $q(x_1)$ has the form 
\begin{equation} 
\frac{1}{Z} \int_{x=-1}^{x_1} 
\pa{1-x^2}^{(d-3)/2}
\int_{y \in \mathcal{S}^{d-2}} \left((1+ D_{11}/n) x^2 + (1-x^2)\sum_{i=2}^d (1+D_{ii}/n) y^2_i\right)^n \,d\mathcal{S}^{d-2}(y) \dx  
\label{eq:integral}
\end{equation} 
If we can evaluate this integral in time $\mbox{poly}(n,d)$, we can sample from $q(x_1)$ by using inverse transform sampling. 

Expanding the term inside the inner integral, \eqref{eq:integral} can be rewritten as 
\begin{align*}
    \frac{1}{Z} \sum_{k=0}^n \binom{n}{k} \int_{x=-1}^{x_1} \left((1+ D_{11}/n) x^2\right)^{n-k} (1-x^2)^{k+(d-3)/2} \int_{y \in \mathcal{S}^{d-2}} (y^\top (I_{d-1}+D_{-1}/n) y)^k \,d\mathcal{S}^{d-2}(y) \dx
\end{align*}
where $D_{-1}$ is obtained from $D$ by deleting the first row and column. By Lemma \ref{eq:integralquadratic}, we can calculate each of the integrals 
 $\int_{y \in \mathcal{S}^{d-2}} (y^\top  (I_{d-1}+D_{-1}/n) y)^k $
in time $\mbox{poly}(n,d)$. 
Also by Lemma \ref{eq:integralquadratic}, $Z$ can be calculated in $\mbox{poly}(n,d)$.

Hence, it remains to show we can approximate in time $\mbox{poly}(n,d)$ an integral of the type 
\begin{equation}
\int_{x=-1}^{x_1} x^{2(n-k)} (1-x^2)^{k+(d-3)/2} \dx. \label{eq:integral2} 
\end{equation} 
We can do this in polynomial time by the trigonometric substitution $x=\sin\te$ and explicitly computing the integral.

Towards showing part 2, we compute the marginal distribution by using Bayes's theorem and making the change of variables $x_{-1} = y\sqrt{1-x_1^2}$, $y=(y_2,\ldots, y_d)\in \mathcal S^{d-2}$,
\begin{align}
\nonumber 
q(x_{-1} |x_1) &\propto q(x_1,x_{-1}) \\ 
\nonumber 
&= \left((1+ D_{11}/n) x^2_1 + \sum_{i=2}^n (1+D_{ii}/n) x^2_i\right)^n \\ 
\nonumber 
&= \left((1+ D_{11}/n) x_1^2 + \sum_{i=2}^n (1-x_1^2) (1+D_{ii}/n) y^2_i\right)^n \\
&= \left(y^\top \left(\left(x_1^2 (1+D_{11}/n) + (1-x_1^2)\right) I_{d-1} +  (1-x_1^2) D_{-1}/n \right) y\right)^n.
\nonumber
\end{align}
The last expression has the form  
$\left(y^\top  (1+\wt{D}/n) y\right)^n$, for 
\begin{align} \label{e:cond-D}
\wt D = x_1^2 D_{11} I_{d-1} + (1-x_1^2) D_{-1},
\end{align}
which is diagonal. Thus, we can apply the same sampling procedure recursively to $\wt D$.
\end{proof} 

\begin{proof}[Proof of Theorem~\ref{t:main-poly}]
As noted, we have reduced to the case of diagonal $D$ with minimum eigenvalue $D_{\min}=0$. 
Let $n=D_{\max}^2$. From Theorem~\ref{t:xDxn} we can sample from the distribution $q(x)\propto (x^\top Dx)^n$ in time $\poly(D_{\max},d)$. By Lemma~\ref{l:approx}, we have $$\exp(-1/2)\le \max\{p(x)/q(x),q(x)/p(x)\}\le \exp(1/2).$$ We would like to do rejection sampling: accept the sample with probability $Cp(x)/q(x)$, where $C$ is a constant $C\le e^{-1/2}$ to ensure this is always $\le 1$; otherwise generate another sample. 
Averaged over $x$ drawn from $q(x)$, the probability of acceptance is then $C$.

However, we don't have access to the normalized distribution $q(x)$. Instead, we have the unnormalized distributions $q^*(x)=(x^\top (I+D/n)x)^n$ and $p^*(x)=\exp(x^\top Dx)$.
We use the ratio at a particular point $v$ to normalize them. Let $v$ the the eigenvector with eigenvalue $0$. We accept a proposal with probability
\begin{align*}
e^{-1} \fc{p^*(x)}{q^*(x)}
&=e^{-1}\fc{q^*(v)}{p^*(v)}\cdot \fc{p^*(x)}{q^*(x)}
= 
    e^{-1}\fc{q(v)}{p(v)}\cdot \fc{p(x)}{q(x)}
\end{align*}
Using the inequality for $v$ in Lemma~\ref{l:approx}, this fits the above framework with $C=e^{-1}\fc{q(v)}{p(v)}\in [e^{-1},e^{-1/2}]$.
This ensures the probability of acceptance is at least $e^{-1}$.
\end{proof}


\section{Conclusion} 

We presented a Las Vegas polynomial time algorithm for sampling from the Bingham distribution $p(x)\propto \exp(x^\top A x)$ on the unit sphere $\mathcal S^{d-1}$. The techniques are based on a novel polynomial approximation of the pdf which we believe is of independent interest, and should find other applications.

There are several natural open problems to pursue---perhaps the most natural one is how to generalize our techniques for the rank-$k$ case. Can these polynomial expansion techniques be used to sample other probabilities of interest Bayesian machine learning, e.g., posterior distributions in latent-variable models such as Gaussian mixture models? More generally, for what other non-log-concave distributions of practical interest can we design provably efficient algorithms? 

\printbibliography

\appendix

\section{Moment calculations}

For completeness, we present the moment calculations that are used in the proof of our main theorem.

\begin{lem}\llabel{l:gf}
Let $A$ be symmetric PSD, and let $\la_1,\ldots, \la_d$ be its eigenvalues. Then the moment generating function of $z^\top A z$, $z\sim N(0,I_d)$, is
\begin{align*}
    f(x) &= \pa{\prodo id \rc{1-2\la_i x}}^{\rc2}
\end{align*}
\end{lem}
\begin{proof}
Without loss of generality $A$ is diagonal. Then $z^TAz = \sumo ik \la_i z_i^2$. The mgf of $r\sim \chi^2_d$ is $\prc{1-2x}^{\rc2}$. Now use the following two facts:
\begin{enumerate}
    \item If the mgf of $X$ is $M_X$, then the mgf of $aX$ is $M_X(at)$: $M_{aX}(t)=M_X(at)$.
    \item The mgf of a sum of random variables is the product of the mgfs: $M_{X+Y}(t)=M_X(t)M_Y(t)$.
\end{enumerate}
\end{proof}

\begin{cor}[\cite{kan2008moments}]\label{c:xTAx-recursion}
Let $A$ be symmetric PSD. 
Let $S(n) = \frac{1}{n! 2^n} \mathbb{E}_{x \sim N(0,I_d)} (x^T Ax)^n$. Then $S(0)=1$ and for $n\ge 1$,
\begin{equation}
    S(n) = \frac{1}{2n}\sum_{i=1}^n \Tr(A^i) S(n-i).
\end{equation}
This can be calculated in polynomial time by dynamic programming.
\end{cor}
\begin{proof}
Note $\Tr(A^k) = \sumo id \la_i^k$.
The moment generating function in Lemma~\ref{l:gf} satisfies the differential equation
\begin{align*}
    f'(x) & = \sumo id \fc{\la_i}{1-2\la_i x} f(x).
\end{align*}
Matching the coefficient of $x^{n-1}$ gives the equation.
\end{proof}

\begin{cor}\label{c:x2n}
For $n\ge 0$,
\begin{align*}
    \E_{x\sim N(0,I_d)}[\ve{x}^{2n}]
    &= \prodz i{n-1} (d+2i). 
\end{align*}
\end{cor}
For $d=1$, this agrees with the formula $\E_{x\sim N(0,1)}[x^{2n}] = (2n-1)!!$.
\begin{proof}
By Lemma~\ref{l:gf}, the moment generating function of $\ve{x}^2$ is $\pa{1-2x}^{-\fc d2}$. 
Use the binomial series expansion.
\end{proof}

\end{document}